\PassOptionsToPackage{table}{xcolor}
\documentclass{article}

\PassOptionsToPackage{numbers,sort&compress}{natbib}
\usepackage{arxiv-ellis}

\usepackage{amsmath,amsfonts,bm}

\def\1{\bm{1}}

\DeclareMathAlphabet{\mathsfit}{\encodingdefault}{\sfdefault}{m}{sl}
\SetMathAlphabet{\mathsfit}{bold}{\encodingdefault}{\sfdefault}{bx}{n}

\newcommand{\Lcal}{{\mathcal L}}

\newcommand{\Ical}{{\mathcal I}}

\newcommand{\Reals}{{\mathbb{R}}}

\usepackage{algorithm}
\usepackage{algpseudocode}
\usepackage{float}
\usepackage{caption}
\usepackage{amsthm}
\newtheorem{proposition}{Proposition}
\usepackage{array}
\usepackage{tikz}
\usetikzlibrary{arrows.meta}
\usepackage{wrapfig}
\newcolumntype{C}[1]{>{\centering\arraybackslash}p{#1}}

\algrenewcommand\algorithmicindent{0.5em}
\title{Learning Discriminative Geometry for Drifting Models}
\author{\textbf{Doudou Zhang}\textsuperscript{1,2}\quad
  \textbf{Wenwen Hou}\textsuperscript{1,2}\quad
  \textbf{Yilin Chen}\textsuperscript{1,2}\quad
  \textbf{Qi Chen}\textsuperscript{1,2}\thanks{Corresponding author.}\\[7pt]
  {\small\textsuperscript{1}ELLIS Institute Finland\quad
  \textsuperscript{2}Aalto University}}
\hypersetup{pdftitle={Learning Discriminative Geometry for Drifting Models},
            pdfauthor={Doudou Zhang, Wenwen Hou, Yilin Chen, Qi Chen}}

\begin{document}

\maketitle

\begin{abstract}
Recently proposed Drifting Models shift iterative distribution refinement from
inference to training, enabling effective one-step generation. However, their performance
on complex image datasets depends strongly on the representation used to
construct the drifting field: pixel-space drifting performs poorly, whereas
pretrained feature spaces substantially improve sample quality for reasons that
remain unclear. We trace this gap to the discriminative geometry of the representation, which determines sample weighting in kernel density estimation (KDE) and, consequently drift. We introduce \emph{persistent representation learning},
which continuously learns a more discriminative representation geometry as the
generator evolves across batches. We further establish a current-step gradient
equivalence between the KDE ratio loss and drift regression loss under matched
conditions, connecting density-ratio-based generator optimization to empirical
drifting and motivating direct control of the drifting velocity. Across multiple
datasets, our method learns effective discriminative representations directly
from pixels and reduces FID by approximately $82-95\%$ over the original
pixel-space Drifting Models, without pretrained encoders. Adapting pretrained representations and applying velocity clipping provide further gains.
\tcblower
{\sffamily\bfseries Code:} \url{https://github.com/aalto-icl/DiscriminativeGeometryforDrifting}
\end{abstract}

\section{Introduction}
\label{sec:introduction}

Diffusion models~\citep{hoDenoisingDiffusionProbabilistic2020,songDenoisingDiffusionImplicit2020}
and flow-matching methods~\citep{liuFlowStraightFast2022,lipmanFlowMatchingGenerative2022} have
achieved remarkable performance in image~\citep{rombachHighResolutionImageSynthesis2021,esserScalingRectifiedFlow2024}
and video generation~\citep{blattmannStableVideoDiffusion2023,yangCogVideoXTexttoVideoDiffusion2024},
with favorable generalization properties supported by recent theoretical
analyses~\citep{chenGeneralizationVAEDiffusion2024,he2026theory}. However,
high-quality synthesis typically relies on iterative sampling, whose computational
cost grows with model scale, spatial resolution, and video length. This has
motivated extensive research on high-quality one-step
generation~\citep{salimansProgressiveDistillationFast2022,liuInstaFlowOneStep2023,songConsistencyModels2023,fransOneStepDiffusion2025,gengMeanFlowsOnestep2025}.
Drifting Models~\citep{dengGenerativeModelingDrifting} provide a distinct
alternative by shifting iterative distribution refinement from inference to
training. During training, they repeatedly construct a drifting field from real
and generated samples and regress the generator toward the corresponding
displaced targets. Once trained, they generate samples in a single inference step, without iterative sampling.

However, the success of Drifting Models on complex image datasets relies heavily
on powerful pretrained encoders. Constructing the drifting field directly in
pixel space often leads to poor sample quality, whereas constructing it in a
pretrained feature space substantially improves
performance~\citep{dengGenerativeModelingDrifting,zhangDistillingDriftingTransformers2026}.
This empirical contrast suggests that the representation space plays a central
role in drifting optimization, yet how it shapes the optimization dynamics remains
poorly understood. Why do pretrained representations induce more effective
drifting fields? Are fixed external features essential, or can suitable
representations emerge jointly with the generator through the drifting objective
itself? Answering these questions is important both for understanding the
mechanism underlying Drifting Models and for improving their training.

We study this problem by examining how the representation space shapes the
drifting field. Representation-induced distances and local neighborhoods
determine the kernel weights, which in turn govern the local density estimates
and the resulting drift direction. When these relationships are defined by a
fixed pixel-space geometry, the batch KDE density-ratio estimate has limited
ability to distinguish real from generated samples, leading to less informative
supervision for the generator. Pretrained encoders instead induce a more
discriminative representation geometry and empirically yield more effective
drifting fields. This observation suggests that effective drifting requires a
representation that captures sample relationships relevant to the evolving
generative task.

Building on this insight, our method addresses both the learning of
representation geometry and its effect on generator optimization. We introduce
\emph{persistent representation learning}, which continuously updates the
representation throughout training while retaining information from successive
discriminative updates in shared parameters. The geometry used to construct the
drifting field can therefore evolve together with the generator, enabling
Drifting Models to learn effective representations directly from pixels without
relying on fixed pretrained features. To characterize how this learned geometry affects generator updates, we build on
the Wasserstein-gradient-flow interpretation of
drifting~\citep{grettonWassersteinGradientFlow2026} and the flow interpretation
of divergence-based GAN training~\citep{yiMonoFlowRethinkingDivergence2023}.
We show that the spatial gradient of the KDE log-density-ratio score recovers
the empirical drifting field up to scaling. Moreover, the corresponding
density-ratio objective yields the same generator-parameter gradient as drift
regression under matched conditions. This equivalence directly connects the
representation-dependent density-ratio geometry to the generator update. Finally, because the drifting field explicitly defines the velocity of each
generated sample, we introduce velocity clipping to bound its magnitude and
stabilize training as the representation evolves. 

Empirically, the resulting framework substantially improves one-step generation
across both learned and pretrained representation settings. Taken together, our
contributions are as follows:
\begin{itemize}
\item We introduce \emph{persistent representation learning}, which accumulates
discriminative updates across batches and progressively adapts the
representation geometry to the evolving generator, supporting both learning directly from pixels and adapting pretrained feature encoders.

\item We establish an algorithmic connection between density-ratio-based
generator optimization and empirical drifting. For Gaussian kernels, we show a
current-step gradient equivalence between the KDE ratio loss and the drift regression loss under matched conditions.

\item We characterize how the local regularity of the density-ratio potential
controls the induced drifting velocity, motivating direct velocity clipping to
stabilize training as the representation evolves.

\item Experiments show that our method reduces FID by approximately $82-95\%$ over the original pixel-space Drifting Models across multiple
datasets without pretrained encoders, with further gains from adapting
pretrained representations and applying velocity clipping. Diagnostic analysis
further reveals how representation learning affects drift alignment and KDE discrimination.
\end{itemize}

\section{Related Work}
\label{sec:related_work}

\textbf{Drifting Models.} Generative Modeling via Drifting~\citep{dengGenerativeModelingDrifting}
constructs a vector field through nonparametric interactions between real and
generated samples, then regresses toward targets displaced by this field to
absorb iterative distribution optimization into a one-step generator.
\citet{turanGenerativeDriftingSecretly2026} connect
Gaussian-kernel drifting fields to score differences between KDE-smoothed
distributions. \citet{laiUnifiedViewScoreBased2026} relate Gaussian- and Laplace-kernel drifting to score-based diffusion models through
Tweedie's formula. \citet{zhangDistillingDriftingTransformers2026} extend Drifting Models to
pretrained representation spaces for one-step distillation and report that
training from scratch without pretrained features remains open.
W-Flow~\citep{hanOneStepGenerativeModeling2026} constructs a drifting model
from the Wasserstein gradient flow of the Sinkhorn divergence and compresses
the resulting particle evolution into a one-step generator. Other studies
improve the computational efficiency of drifting or analyze its finite-particle
convergence and stability~\citep{falahatiDriftXpressFasterDrifting2026,
balasubramanianFiniteParticleConvergenceRates2026,
leeIdentifiabilityStabilityGenerative2026}. Despite these advances, the role of representation in drifting remains
underexplored: existing approaches typically rely on a fixed representation,
leaving open how the discriminative representation can adapt
during training.

\textbf{Generative Adversarial Networks.} GANs train a generator against a discriminator that distinguishes real from
generated samples~\citep{goodfellowGenerativeAdversarialNetworks2014a}.
Variational formulations such as $f$-GAN show that the optimal discriminator
encodes a density ratio through a prescribed transformation
~\citep{nowozinFGANTrainingGenerative2016}, while
MonoFlow~\citep{yiMonoFlowRethinkingDivergence2023} interprets the resulting
generator direction through Wasserstein gradient flows. KDD-GAN and Projected
GAN demonstrate the importance of the representation used for discrimination,
using learned and fixed pretrained feature spaces, respectively
~\citep{lemkhenterGenerativeAdversarialLearning2021,
sauerProjectedGANsConverge2021}. Lipschitz constraints, implemented through
methods such as gradient penalties and spectral normalization, further control
the discriminator gradients received by the generator
~\citep{gulrajaniImprovedTrainingWasserstein2017,
miyatoSpectralNormalizationGenerative2018}. These works highlight two recurring mechanisms in adversarial training:
persistent discriminator learning determines how the discrepancy between real
and generated distributions is represented, while regularity constraints
control the spatial gradients through which this discrepancy acts on the
generator.

\textbf{Wasserstein Gradient Flows.} Wasserstein gradient flows describe the steepest descent of an energy
functional over probability distributions under the geometry of optimal
transport~\citep{santambrogioEuclideanMetricWasserstein2017}. In generative modeling, this perspective defines a velocity field
that transports the generated distribution toward the real data distribution~\citep{arjovskyWassersteinGAN2017}.
\citet{grettonWassersteinGradientFlow2026} connect
Gaussian-kernel drifting
to the limiting point of a Wasserstein gradient flow associated with the KL
divergence between Parzen-smoothed distributions, and analyze extensions to
other distributional objectives. These analyses characterize the distribution-level dynamics underlying
drifting, but leave the finite-batch construction of the drifting field
and the regularity of its induced generator update as separate algorithmic
questions.

Our work connects these three lines of research by studying how the
discriminative representation shapes the empirical
drifting field and the resulting generator update. We introduce persistent
representation learning to adapt the representation across batches
instead of using a fixed or pretrained representation. We further show that
the KDE density-ratio potential connects the drifting field to GAN generator
updates, which motivates directly clipping the drifting velocity to stabilize
training as the representation evolves.

\section{Background}
\label{sec:background}

\textbf{Drifting Models.} Drifting Models~\citep{dengGenerativeModelingDrifting} construct a drifting field
from attractive interactions with real samples and repulsive interactions with
generated samples, and use this field to iteratively update the generator. Let
$p$ denote the real data distribution. A generator $G_\theta:\Reals^{d^\prime}\rightarrow\Reals^d$ maps
some latent $z$ sampled from some easy-to-sample distribution $\mu$ (typically Gaussian) with $z\sim \mu$ to the generated data $\hat{x}=G_\theta(z)$. 
This process induces the generated data distribution $q_\theta=G_{\theta}\#\mu$.
Let $k_\tau:\mathbb R^d\times\mathbb R^d\to(0,\infty)$ be a positive
kernel measuring similarity between samples, with bandwidth parameter
$\tau>0$. We use the Gaussian kernel
\begin{equation}
k_\tau(x,y)
=(\pi\tau)^{-d/2}
e^{-\|x-y\|_2^2/\tau}.
\label{eq:gaussian_kernel}
\end{equation}
where $d$ is the sample dimension. Its normalization constant cancels in
the normalized weights below. A \emph{drifting field} is a vector field
$V_{p,q_{\theta}}:\Reals^d\to\Reals^d$, determined by the pair
$(p,q_{\theta})$, that assigns to every point $x$ a displacement pointing
toward the real data and away from the generated data. It is formed as the
difference of two kernel-weighted local means, one taken over real samples and
one over generated samples, and in practice it is estimated from the samples in
the current batch. Given real samples $\{\tilde{x}_j\}_{j=1}^{N^+}$ and
generated samples $\{\hat{x}_i\}_{i=1}^{N^-}$, we have the empirical
drifting field\footnote{For simplicity, we omit that same distribution KDE estimates exclude the evaluation point.}:
\begin{equation}
\widehat V_{p,q_\theta}(x)
=
\frac{\sum_{j=1}^{N^+}k_\tau(x,\tilde{x}_j)(\tilde{x}_j-x)}
{\sum_{j=1}^{N^+}k_\tau(x,\tilde{x}_j)}
-
\frac{\sum_{i=1}^{N^-}k_\tau(x,\hat{x}_i)(\hat{x}_i-x)}
{\sum_{i=1}^{N^-}k_\tau(x,\hat{x}_i)}.
\label{eq:empirical_drift}
\end{equation}
The first term is the attractive component pulling $x$ toward nearby real
samples, and the second the repulsive component pushing it away from nearby
generated samples. The generator regresses toward
the displaced stop-gradient target through the following drift regression loss
\begin{equation}
\mathcal L^{\text{Drift}}(\theta)
=
\mathbb E_{z\sim\mu}
\left[
\left\|
G_\theta(z)-
\operatorname{sg}\!\left(
G_\theta(z)+\widehat V_{p,q_\theta}(G_\theta(z))
\right)
\right\|_2^2
\right],
\label{eq:drifting_loss}
\end{equation}
where the field is constructed from the current batch and held fixed
during the generator update.

For the Gaussian kernel in Eq.~(\ref{eq:gaussian_kernel}), the drift admits
a score-difference interpretation through Parzen-smoothed densities.
We define the smoothed densities of the real data distribution $p$
and the generated data distribution $q_{\theta}$ as
\begin{equation*}
p_\tau(x)=\mathbb E_{\tilde{x}\sim p}[k_\tau(x,\tilde{x})],
\qquad
q_\tau^{\theta}(x)=\mathbb E_{\hat{x}\sim q_{\theta}}[k_\tau(x,\hat{x})].
\end{equation*}

The population drift is their score difference
\(
V_{p,q_{\theta}}(x)=\frac{\tau}{2}
\left(\nabla\log p_\tau(x)-\nabla\log q_\tau^{\theta}(x)\right).
\)
This is a smoothed approximation of the Wasserstein gradient flow associated
with $\mathrm{KL}(q_{\theta}\|p)$~\citep{grettonWassersteinGradientFlow2026}, whose
velocity is
$V^{\mathrm{KL}}_{p,q_{\theta}}(x)=\nabla\log p(x)-\nabla\log q_{\theta}(x).$

\section{Methodology}
\label{sec:methodology}

\begin{wrapfigure}{r}{0.55\textwidth}
\vspace{-15pt}
\centering
\begin{minipage}{\linewidth}
\hrule
\captionof{algorithm}{Persistent Representation for Drifting Training with
Velocity Clipping}
\label{alg:trainable_encoder_drifting}
\vspace{2pt}
\hrule
\vspace{2pt}
\small
\textbf{Input:} Generator $G_\theta$, encoder $E_\varphi$, real data distribution
$p$, noise distribution $\mu$, kernel $k_\tau$, batch size $B$, encoder steps
$K$, velocity threshold $V_{\max}$, and learning rates
$\gamma_E,\gamma_G$
\begin{algorithmic}[1]
\Repeat
  \For{$k=1,\ldots,K$}
    \State Sample fresh $\{z_i\}_{i=1}^{B}\sim\mu$ and
      $\{\tilde{x}_j\}_{j=1}^{B}\sim p$
    \State $\hat{x}_i\gets\operatorname{sg}(G_\theta(z_i))$,
      $u_i\gets E_\varphi(\hat{x}_i)$, $v_j\gets E_\varphi(\tilde{x}_j)$
    \State Evaluate $s_\varphi^{\mathrm{KDE}}(\tilde{x}_j)$ and
      $s_\varphi^{\mathrm{KDE}}(\hat{x}_i)$
    \State Compute $\widehat{\Lcal}_{\varphi}^D$ using
      Eq.~(\ref{eq:encoder_kde_objective}) and update $\varphi$
  \EndFor
  \State Sample fresh $\{z_i\}_{i=1}^{B}\sim\mu$ and
    $\{\tilde{x}_j\}_{j=1}^{B}\sim p$
  \State $u_i\gets E_\varphi(G_\theta(z_i))$ and
    $v_j\gets E_\varphi(\tilde{x}_j)$
  \State Construct $\widehat V_\varphi(u_i)$
  \State $\widehat V_{\varphi,\mathrm{clip}}(u_i)\gets
    \widehat V_\varphi(u_i)/
    \max(1,\|\widehat V_\varphi(u_i)\|_2/V_{\max})$
  \State $\widetilde u_i\gets
    \operatorname{sg}(u_i+\widehat V_{\varphi,\mathrm{clip}}(u_i))$
  \State $\widehat{\mathcal L}_{\varphi,\theta}^{\text{Drift}}
    \gets B^{-1}\sum_{i=1}^{B}\|u_i-\widetilde u_i\|_2^2$
    \Comment{clipped target}
  \State Update $\theta$ using this loss
\Until{converged}
\end{algorithmic}
\hrule
\end{minipage}
\vspace{-15pt}
\end{wrapfigure}

In Drifting Models, the representation induces the geometry used by the KDE: feature-space distances determine which real and generated samples receive high kernel weights and how they contribute to the drifting update. With a fixed representation, the KDE and drifting field still evolve with the generator distribution, but the geometry used to measure and localize discrepancies remains unchanged. We therefore train the encoder alongside the generator, retaining its parameters across batches while rebuilding the KDE from each current batch. We first characterize how the representation geometry shapes the drifting update, then introduce alternating encoder and generator updates, and finally apply velocity clipping to control the magnitude of the feature-space drift. Algorithm~\ref{alg:trainable_encoder_drifting} summarizes the complete training procedure.

\subsection{KDE-Ratio and Drifting Updates in Representation Space} \label{sec:density_ratio_drift}

We formulate the generator update directly in the representation space used to construct the drifting field, where the corresponding pixel-space update is recovered by taking \(E_\varphi\) to be the identity map. Let
$E_\varphi:\mathbb R^d\rightarrow\mathbb R^{\tilde{d}}$ be an encoder.
We use $k_\tau$ in the dimension of its arguments; in representation space,
the Gaussian normalization in Eq.~(\ref{eq:gaussian_kernel}) is
$(\pi\tau)^{-\tilde d/2}$. We set $N^+=N^-=B$ below.
For a batch of $B$ real samples $\{\tilde{x}_j\}_{j=1}^B$ and generated
samples $\{\hat{x}_i\}_{i=1}^B$, where
$\hat{x}_i=G_\theta(z_i)$ and $z_i\sim\mu$, we hold $\varphi$ fixed during
the generator update.
Their corresponding representations are
$v_j = E_\varphi(\tilde{x}_j)$ and $u_i = E_\varphi(\hat{x}_i)$.
We then construct batch KDEs in the representation space with
$
\widehat p_{\varphi,\tau}(v)
=
\frac{1}{B}\sum_{j=1}^{B} k_\tau(v,v_j)
$ and 
$
\widehat q_{\varphi,\tau}^{\theta}(v)
=
\frac{1}{B}\sum_{i=1}^{B} k_\tau(v,u_i).
$

For any input \(x \in \mathbb{R}^d\), we define its KDE-based log-density-ratio score ~\citep{lemkhenterGenerativeAdversarialLearning2021} in the representation space induced by $E_{\varphi}$ as
\begin{equation}
s_\varphi^{\mathrm{KDE}}(x):=\log \widehat p_{\varphi,\tau}\!\left(E_\varphi(x)\right) - \log \widehat q_{\varphi,\tau}^\theta\!\left(E_\varphi(x)\right).
\end{equation}
This score can serve as a real-versus-generated classification logit
under equal class priors, as detailed in Section~\ref{sec:persistent_representation}.
We define the KDE ratio loss
$
\widehat{\mathcal L}_{\varphi, \theta}^{\text{KDE}}
:=-\frac{\tau}{B}\sum_{i=1}^B
s_\varphi^{\mathrm{KDE}}(G_\theta(z_i))$. 
For a fixed encoder, the feature-space drifting field is
\begin{equation}
\widehat V_\varphi(u)
=\frac{\sum_{j=1}^{B}k_\tau(u,v_j)(v_j-u)}
{\sum_{j=1}^{B}k_\tau(u,v_j)}
-\frac{\sum_{i=1}^{B}k_\tau(u,u_i)(u_i-u)}
{\sum_{i=1}^{B}k_\tau(u,u_i)}\,,
\label{eq:feature_space_drift}
\end{equation}
Then the drift regression loss in the representation space is
\begin{equation}
\widehat{\mathcal L}_{\varphi, \theta}^{\text{Drift}}:=\frac{1}{B}\sum_{i=1}^{B}
\left\|E_{\varphi}(G_\theta(z_i))-\operatorname{sg}(u_i+\widehat V_\varphi(u_i))\right\|_2^2.
\label{eq:persistent_representation_generator_loss}
\end{equation}
During the generator update, the encoder parameters \(\varphi\) are held fixed, while gradients are still propagated through \(E_{\varphi}\) w.r.t. its input. Consequently, the feature-space drifting direction is pulled back to the generator parameters through the Jacobians of the encoder and generator. The following proposition makes this connection explicit by relating the generator gradient of the KDE-based log-density-ratio score to the feature-space drifting field.

\begin{proposition}[Equivalence of representation-space generator gradients]
\label{prop:kde_drifting_gradient}
Let $B\ge1$ and $\tau>0$. For each sampled input $z_i$, let
$\hat{x}_i=G_\theta(z_i)$ and $u_i=E_\varphi(\hat{x}_i)$.
Assume that $G_\theta$ is differentiable w.r.t. $\theta$ at each
$z_i$ and that $E_\varphi$ is differentiable w.r.t. its input at
each $\hat{x}_i$. Let $J_E(x)=\partial E_\varphi(x)/\partial x$ and
$J_G(z)=\partial G_\theta(z)/\partial\theta$.
Use the Gaussian kernel in Eq.~(\ref{eq:gaussian_kernel}), with the same
bandwidth and feature anchors in the KDE ratio loss
$\widehat{\mathcal L}_{\varphi,\theta}^{\text{KDE}}$ and the drift regression loss
$\widehat{\mathcal L}_{\varphi,\theta}^{\text{Drift}}$ defined above.
Hold $\varphi$ and the feature anchors fixed, differentiating the KDE ratio
loss only through the generated queries $E_\varphi(G_\theta(z_i))$.
With no velocity clipping, the gradients evaluated at the current parameters
used to construct the detached regression targets satisfy
\begin{equation*}
\nabla_\theta\widehat{\mathcal L}_{\varphi,\theta}^{\text{KDE}}
=\nabla_\theta\widehat{\mathcal L}_{\varphi,\theta}^{\text{Drift}}
=-\frac{2}{B}\sum_{i=1}^B
J_G(z_i)^\top J_E(\hat{x}_i)^\top\widehat V_\varphi(u_i).
\label{eq:generator_gradient_equivalence}
\end{equation*}
\end{proposition}

The proof is given in Appendix~\ref{app:proof_kde_drifting}.
Pixel-space drifting is the special case $E_\varphi=\mathrm{Id}$;
Appendix~\ref{app:pixel_kde_drifting} gives its original formulation and
compares the two cases. Appendix~\ref{app:gan_critic_connection}
discusses how this KDE-based classification score relates to density-ratio
estimates induced by GAN critics
\citep{yiMonoFlowRethinkingDivergence2023}.

\textbf{Learning a discriminative geometry for generator supervision.}
Proposition~\ref{prop:kde_drifting_gradient} shows that the same KDE score
used to distinguish real and generated samples also supplies the generator
supervision in unclipped drifting. The encoder changes both the kernel
weights defining this score and the gradients propagated to the generator.
This motivates training $E_\varphi$ to learn a discriminative geometry,
rather than fixing the geometry throughout training. We optimize the KDE
classification objective introduced next to expose current real--generated
discrepancies, then update the generator using the corresponding drift.
The gradient identity does not imply that every representation improves on
pixels, or that stronger batch discrimination guarantees better generation;
Figures~\ref{fig:toy_field_alignment} and
\ref{fig:representation_fid_kde_diagnostics} examine drift alignment,
KDE discrimination, and generation quality empirically.

\subsection{Learning Discriminative Geometry through Persistent Representations}
\label{sec:persistent_representation}

\textbf{Persistent Representation Learning.} Standard drifting constructs sample relationships in a fixed representation,
so the discriminative geometry cannot adapt during
training. We now train $E_\varphi$ and retain its parameters across batches
throughout training. This encoder is used only to define the representation
space for drifting, while $G_\theta$ continues to output images in pixel space. The KDE estimates and drifting field are reconstructed from each
current batch. This allows earlier batch samples to
shape subsequent discriminative geometry without storing them.
Section~\ref{sec:ablations} evaluates persistence in representation against maintaining
persistence in a scalar critic or both. In our from-scratch experiments, $E_\varphi(x):=x+R_\varphi(x)$ with
$R_\varphi:\mathbb{R}^d\to\mathbb{R}^d$ preserves the pixel dimension $d$,
which allows $E_\varphi$ to be initialized exactly as the identity mapping;
$R_\varphi$ is implemented as a lightweight variant of a Dilated Residual
Network (DRN)~\citep{yuDilatedResidualNetworks2017}, with architecture details
in Appendix~\ref{app:representation_architecture}.

\textbf{Learning a Discriminative Representation.} The representation determines the pairwise distances and kernel weights entering the KDE estimates, and therefore directly affects the density-ratio estimate from which the drifting field is subsequently constructed. We train $E_\varphi$ so that the KDE comparison
used by Drifting Models more clearly exposes differences between current real
and generated samples.
Before each generator update, we perform $K$ encoder
updates using fresh real and generated batches, with
$\hat{x}_i=\operatorname{sg}(G_\theta(z_i))$. For these encoder updates, we use
$s_\varphi^{\mathrm{KDE}}$ defined in Section~\ref{sec:density_ratio_drift}.
This score uses no additional trainable scalar head, and its gradient
propagates through the KDE computation to the encoder.

To classify real and generated samples, we use equal class priors, since
each classification batch contains equal numbers of equally weighted samples
from the two classes. By Bayes' rule, the
KDE-based posterior in representation space is
\begin{equation*}
\widehat P(y=\mathrm{real}\mid u)
=\frac{\widehat p_{\varphi,\tau}(u)}
{\widehat p_{\varphi,\tau}(u)+\widehat q_{\varphi,\tau}^{\theta}(u)}
=\sigma\!\left(s_\varphi^{\mathrm{KDE}}(x)\right),
\end{equation*}
where $u=E_\varphi(x)$ and $\sigma(t)=(1+e^{-t})^{-1}$. Thus, the score serves as a
classification logit: positive values favor real samples and negative values
favor generated samples. Using this posterior, we train the encoder with binary
cross-entropy, summed over the two class-wise averages:
\begin{equation}
\widehat{\Lcal}_{\varphi}^D
:=-\frac{1}{B}\sum_{j=1}^{B}
\log\sigma\!\left(s_\varphi^{\mathrm{KDE}}(\tilde{x}_j)\right)
-\frac{1}{B}\sum_{i=1}^{B}
\log\!\left(1-\sigma\!\left(s_\varphi^{\mathrm{KDE}}(\hat{x}_i)\right)\right).
\label{eq:encoder_kde_objective}
\end{equation}
This objective encourages higher KDE scores for real samples and lower
scores for generated samples, with the aim of learning a more discriminative
geometry for constructing the drifting field. 

This phase updates only the
representation. The generator continues to learn from an explicitly
constructed drifting field. After $K$ representation updates, we freeze $\varphi$ and draw fresh
real samples and noise variables. We then construct the feature-space drift
in Eq.~(\ref{eq:feature_space_drift}) and update the generator using
Eq.~(\ref{eq:persistent_representation_generator_loss}), with the velocity
control described below.

\textbf{Controlling the Representation-Induced Drift Velocity.} The density-ratio connection further allows us to characterize the
magnitude of the induced drift. In particular, the local regularity of the
density-ratio potential bounds the corresponding drift velocity. We use this
relation to motivate a direct constraint on the velocity applied in the
generator update.

For the Gaussian kernel, let
$\ell_\varphi(u)=\log\widehat p_{\varphi,\tau}(u)-\log\widehat q_{\varphi,\tau}^{\theta}(u)$
denote the feature-space log-density ratio, so that
$s_\varphi^{\mathrm{KDE}}(x)=\ell_\varphi(E_\varphi(x))$.
With the anchors fixed, the identity underlying
Proposition~\ref{prop:kde_drifting_gradient} gives
$\widehat V_\varphi(u)=\frac{\tau}{2}\nabla_u\ell_\varphi(u)$.
Thus, learning the representation changes both the direction and magnitude
of the drift. If $\|\nabla_u\ell_\varphi(u)\|_2\le L$ in a region, then
$\|\widehat V_\varphi(u)\|_2\le\tau L/2$ there. This motivates directly
bounding the feature-space displacement used in the drift regression target.
The analogy with Lipschitz control of GAN critics is discussed in
Appendix~\ref{app:gan_critic_connection}.

This relation suggests two forms of control: spectral normalization of the
encoder indirectly constrains local representation changes, whereas a direct
constraint acts on the feature-space drift received by the generator. We use
the latter and compare the two approaches in
Section~\ref{sec:ablations}. Given a velocity
budget $V_{\max}>0$, we define
\begin{equation}
\widehat V_{\varphi,\mathrm{clip}}(u)
=
\widehat V_\varphi(u)/
\max\!\left(1,\|\widehat V_\varphi(u)\|_2/V_{\max}\right).
\label{eq:l2_drift_projection}
\end{equation}
This radial projection maps zero drift to zero, preserves the direction of
nonzero drift, and guarantees
$\|\widehat V_{\varphi,\mathrm{clip}}(u)\|_2\le V_{\max}$. The generator subsequently uses
$\widetilde u=\operatorname{sg}(u+\widehat V_{\varphi,\mathrm{clip}}(u))$ as its
regression target. This bound controls the feature-space displacement; the
generator parameter gradient also depends on the encoder and generator
Jacobians.

\begin{figure*}[t]
\centering
\includegraphics[width=\textwidth]{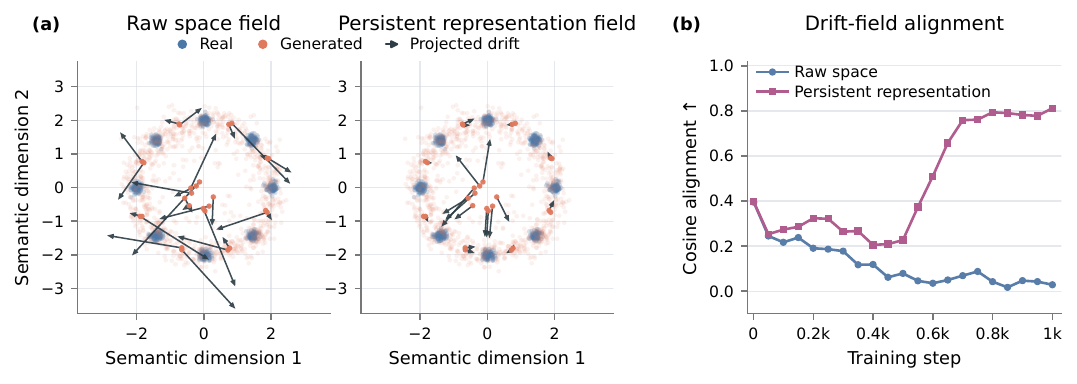}
\caption{\textbf{Persistent Representation Learning Improves Drift Alignment.}
\textbf{(a)} Drift field constructed in the raw space is disorganized by the nuisance coordinates, while the
persistent representation field aligns with the ring shaped target
structure. \textbf{(b)} Persistent representation space field reaches
substantially higher alignment with the reference field than raw space as training progresses.}
\label{fig:toy_field_alignment}
\end{figure*}

\section{Experiments}
\label{sec:experiments}

To evaluate the proposed method and investigate the factors that shape its performance, we organize our experiments around three research questions. 
\textbf{RQ1:} How does representation geometry shape drift estimation, and what are its consequences for KDE discrimination and sample quality?
\textbf{RQ2:} Can persistent representation learning enable effective drifting
without pretrained features?
\textbf{RQ3:} How do persistence in different components and drift velocity control affect sample quality?
We address these questions through mechanistic analysis, benchmark
evaluations, and ablation studies.

Following DriftXpress~\citep{falahatiDriftXpressFasterDrifting2026}, we use
U-Net~\citep{ronnebergerUNetConvolutionalNetworks2015} as the generator and conduct experiments on SVHN~\citep{netzerReadingDigitsNatural2011},
CIFAR-10~\citep{krizhevskyLearningMultipleLayers2009}, and CIFAR-100~\citep{krizhevskyLearningMultipleLayers2009}. We report
FID and Inception Score (IS) on 50K generated images, with detailed architectures, optimization
settings, and training configurations provided in
Appendix~\ref{sec:implementation_details}. We use \emph{standard drifting}
to denote drifting directly in pixel space. Our experiments
focus on isolating the effect of persistent representation learning relative
to standard drifting, rather than reproducing computationally intensive large-scale
results in ~\citet{dengGenerativeModelingDrifting}.

\begin{figure*}[!t]
\centering
\includegraphics[width=\textwidth]{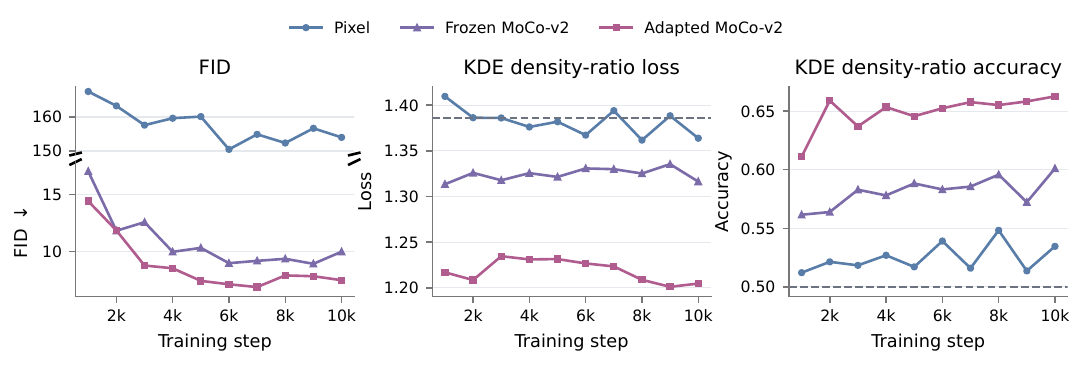}
\caption{\textbf{Persistent Representations Improve Discrimination and Generation.}
Pixel-space KDE discrimination remains near chance, while a frozen MoCo-v2
representation improves both discrimination and sample quality, and persistent
adaptation improves them further. Dashed lines indicate the loss for indistinguishable distributions and chance-level accuracy, respectively.}
\label{fig:representation_fid_kde_diagnostics}
\end{figure*}

\subsection{Mechanistic Analysis}
\label{sec:mechanistic_analysis}

\textbf{Learning Discriminative Geometry through Persistent Representations.} To directly evaluate how representation geometry affects drift estimation, we
use a controlled toy example with eight Gaussian modes in a known
two-dimensional semantic space and 30 Gaussian nuisance dimensions; full
settings are given in Appendix~\ref{sec:implementation_details}.
Figure~\ref{fig:toy_field_alignment}(a) compares drift fields computed on the
same generated samples, with kernel weights determined in the raw space and learned representation space.
The raw-space field varies sharply across nearby samples, whereas the learned representation produces a more coherent field around the target modes.
Figure~\ref{fig:toy_field_alignment}(b) further shows substantially higher
cosine alignment with a reference KDE drift field computed using only the two
semantic coordinates, where alignment is measured between the projected
estimated drift field and the reference drift field at the same generated samples.
These results illustrate how a persistent representation can improve batch
drift estimation when raw distances are affected by nuisance variation.

\begin{figure}[!t]
\centering
\includegraphics[width=0.90\textwidth]{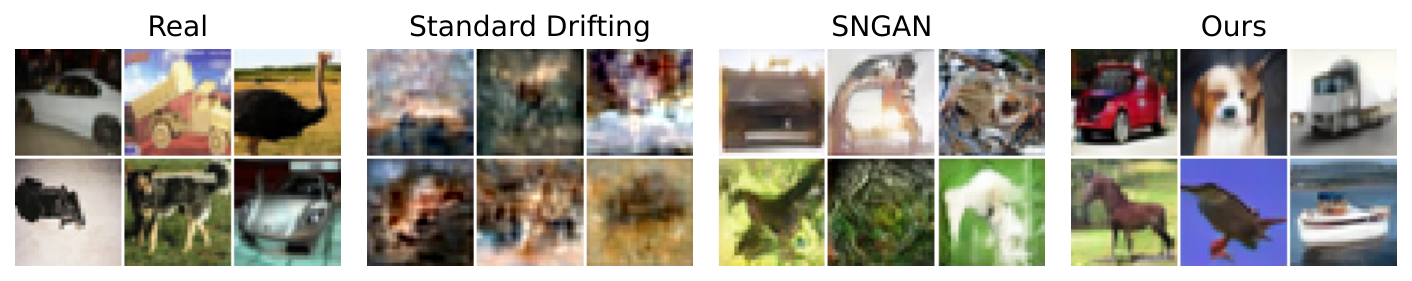}
\caption{\textbf{Qualitative Comparisons on CIFAR-10.} Our method produces substantially clearer and more recognizable
samples than standard drifting, with more coherent object structure than
SNGAN.}
\label{fig:generated_samples_comparison}
\vspace{-0.5cm}
\end{figure}

\textbf{Effect of Better Geometry for Discrimination with Persistent Representation.} We next examine how representation geometry affects KDE discrimination and
sample quality on image data. We compare three representations for constructing
the drifting field on CIFAR-10:
raw pixels, a frozen MoCo-v2 encoder
~\citep{heMomentumContrastUnsupervised2020,chenImprovedBaselinesMomentum2020},
and the same encoder with its final residual block adapted by persistent
representation learning. Alongside FID, we report the encoder objective
$\widehat{\Lcal}_{\varphi}^{D}$ in Eq.~(\ref{eq:encoder_kde_objective}) as the
KDE classification loss and compute the corresponding classification accuracy by
assigning positive scores to real samples and negative scores to generated
samples.

Figure~\ref{fig:representation_fid_kde_diagnostics} shows that pixel-space
drifting remains ineffective and KDE density-ratio
accuracy close to chance; Appendix~\ref{app:pixel_bandwidth_sweep} confirms that
this is not due to an untuned kernel bandwidth. A frozen MoCo-v2 representation
improves the FID, together with higher KDE density-ratio accuracy
and lower loss, while adapting its final residual block further improves the
FID and discrimination. These results show that the representation significantly affects both KDE discrimination and sample quality.

\begin{figure*}[!t]
\centering
\includegraphics[width=\textwidth]{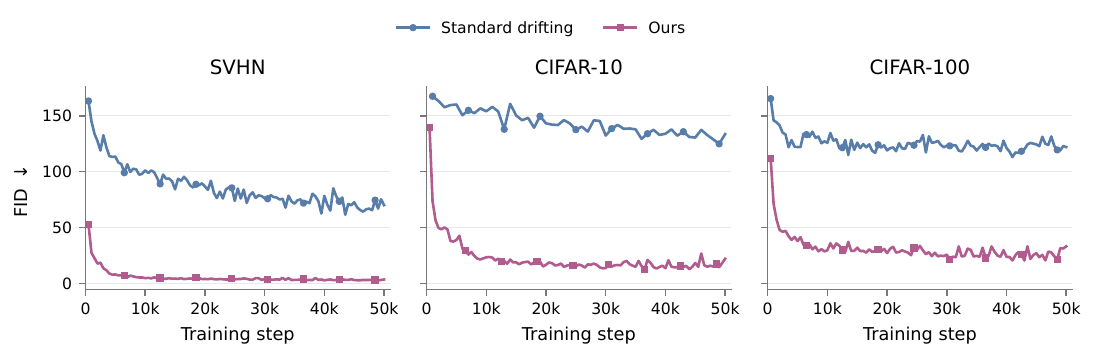}
\caption{\textbf{FID over Training Steps.}
Our method converges to substantially lower FID than standard drifting on
SVHN, CIFAR-10, and CIFAR-100.}
\label{fig:main_results_triptych}
\vspace{-0.5cm}
\end{figure*}

\subsection{Main Benchmark Results}
\label{sec:main_results}

We compare our method with standard drifting to evaluate whether it enables
effective drifting without pretrained features. We additionally include
SNGAN~\citep{miyatoSpectralNormalizationGenerative2018} as an adversarial
reference, using a widely used open-source
implementation.\footnote{\href{https://github.com/christiancosgrove/pytorch-spectral-normalization-gan}{https://github.com/christiancosgrove/pytorch-spectral-normalization-gan}}

\textbf{Quantitative Results.}
Our method consistently improves standard drifting across all three datasets.
Figure~\ref{fig:main_results_triptych} and Table~\ref{tab:main_results} show
substantial improvements in both FID and IS on SVHN, CIFAR-10, and CIFAR-100.
Our method also achieves better performance than SNGAN across all three datasets.

\textbf{Qualitative Results.}
Figure~\ref{fig:generated_samples_comparison} shows that our method produces
substantially clearer and more recognizable samples than standard drifting,
with more coherent object structure than SNGAN. Appendix~\ref{app:additional_qualitative_results} provides additional SVHN
samples, Appendix~\ref{app:learned_geometry_evolution} tracks the learned
representation across checkpoints, and Appendix~\ref{app:training_cost}
confirms the gain at equal training time.

\begin{table}[H]
\centering
\caption{\textbf{Quantitative Comparisons.}
Our method substantially improves both FID and IS over standard drifting on
SVHN, CIFAR-10, and CIFAR-100, and also outperforms SNGAN.}
\label{tab:main_results}
\small
\setlength{\tabcolsep}{4pt}
\begin{tabular}{lrrrrrr}
\toprule
& \multicolumn{2}{c}{SVHN}
& \multicolumn{2}{c}{CIFAR-10}
& \multicolumn{2}{c}{CIFAR-100} \\
\cmidrule(lr){2-3}\cmidrule(lr){4-5}\cmidrule(lr){6-7}
Method & IS $\uparrow$ & FID $\downarrow$
       & IS $\uparrow$ & FID $\downarrow$
       & IS $\uparrow$ & FID $\downarrow$ \\
\midrule
Standard drifting & 2.91 & 61.56 & 3.18 & 124.94 & 3.46 & 113.14 \\
Ours           & \textbf{3.24} & \textbf{2.88}
               & \textbf{8.69} & \textbf{12.81}
               & \textbf{8.52} & \textbf{20.61} \\
\midrule
SNGAN & 3.04 & 13.16 & 7.22 & 28.59 & 7.96 & 30.33 \\
\bottomrule
\end{tabular}
\end{table}

\subsection{Ablation Studies}
\label{sec:ablations}

\textbf{Effect of Persistence in Different Components.} To study where cross-batch persistence is most effective, we compare four
configurations: no persistent component, persistent representation only,
persistent scalar critic only, and persistence in both the representation and
the scalar critic. As shown in Figure~\ref{fig:persistence_placement_ablation}, within 10k training
steps on CIFAR-10~\citep{krizhevskyLearningMultipleLayers2009}, persistent
representation substantially improves performance over no persistence, whereas introducing persistence in the
scalar critic leads to markedly worse results.

\begin{figure*}[!t]
\centering
\includegraphics[width=0.96\textwidth]{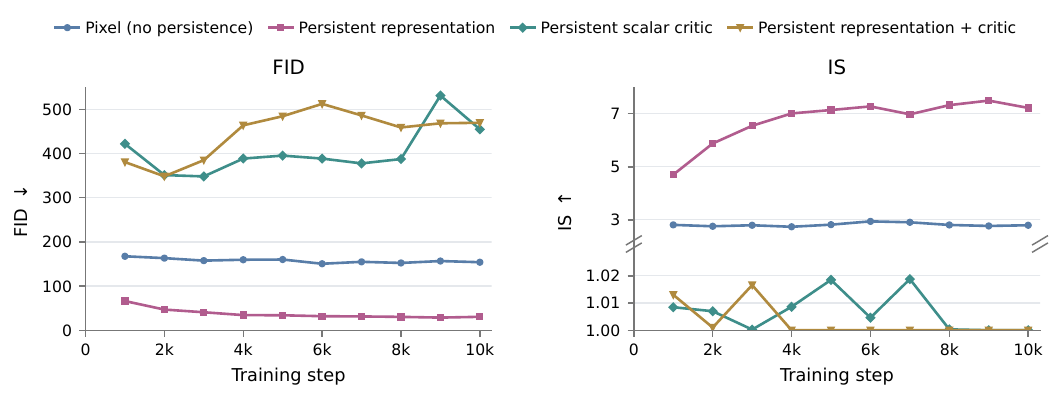}
\caption{\textbf{Effect of Persistence in Different Components.}
Among the four configurations, placing persistence only in the representation
achieves substantially lower FID and higher IS than the other configurations,
whereas introducing a persistent scalar critic degrades performance.}
\label{fig:persistence_placement_ablation}
\end{figure*}

\textbf{Effect of Lipschitz Regularization.} Finally, we study whether Lipschitz regularization should be applied to the
trainable residual encoder or directly to the drifting velocity. The former
uses spectral normalization for every convolution in the encoder, whereas the
latter clips the velocity that updates the generator. Table~\ref{tab:velocity_budget_ablation}
reports the best IS and FID within 10k steps on CIFAR-10. Spectral normalization
degrades performance relative to the unconstrained model, whereas velocity clipping improves both
metrics for all evaluated budgets, with $V_{\max}=0.005$ achieving the best performance. These results support constraining the explicit
velocity rather than the encoder; Appendix~\ref{app:velocity_clip_before_clip}
shows the pre-clipping velocity distribution and the locations of the evaluated
budgets. Table~\ref{tab:persistence_velocity_factorial} further separates the effects of
persistent representation and velocity clipping. Persistent representation 
provides the dominant improvement, while velocity clipping yields additional
gains.

\begin{table*}[!t]
\centering
\begin{minipage}[t]{0.51\textwidth}
\captionof{table}{\textbf{Effect of Lipschitz Regularization.}
Velocity clipping outperforms spectral normalization.}
\label{tab:velocity_budget_ablation}
\centering
\small
\begin{tabular}{lrr}
\toprule
Lipschitz regularization & IS $\uparrow$ & FID $\downarrow$ \\
\midrule
Unconstrained                          & 7.50 & 27.41 \\
Encoder spectral normalization         & 7.23 & 31.82 \\
\midrule
\multicolumn{3}{l}{Velocity clipping} \\
\quad $V_{\max}=1$       & 7.53 & 25.72 \\
\quad $V_{\max}=0.5$     & 7.75 & 22.19 \\
\quad $V_{\max}=0.05$    & 7.66 & 23.47 \\
\quad $V_{\max}=0.005$   & \textbf{7.80} & \textbf{21.57} \\
\quad $V_{\max}=0.0005$  & 7.67 & 22.85 \\
\bottomrule
\end{tabular}
\end{minipage}\hfill
\begin{minipage}[t]{0.47\textwidth}
\captionof{table}{\textbf{Factorial Ablation.}
Persistent representation provides most of the improvement, while velocity
clipping further improves performance. Velocity clipping uses $V_{\max}=1.5$ in pixel space.}
\label{tab:persistence_velocity_factorial}
\centering
\small
\begin{tabular}{ccrr}
\toprule
\shortstack{Persistence\\representation} &
\shortstack{Velocity\\clipping} &
IS $\uparrow$ & FID $\downarrow$ \\
\midrule
$\times$     & $\times$     & 2.94 & 150.49 \\
$\times$     & $\checkmark$ & 3.44 & 126.33 \\
$\checkmark$ & $\times$     & 7.50 & 27.41 \\
$\checkmark$ & $\checkmark$ & \textbf{7.80} & \textbf{21.57} \\
\bottomrule
\end{tabular}
\end{minipage}
\end{table*}

\section{Conclusion}
\label{sec:conclusion}

We studied whether Drifting Models can learn the representation geometry
required for effective drifting without relying on fixed pretrained features.
Our analysis identifies a representation
bottleneck: the geometry determines kernel weights and drift directions, while
fixed pixel space provides weak discrimination performance.
We addressed this bottleneck with persistent representation learning,
which continuously updates the representation while retaining information from
successive discriminative updates, so that the geometry used to construct the
drifting field evolves together with the generator. We further connect the KDE
density-ratio potential to the empirical drifting field. 
This connection relates the local regularity of the potential to
the drifting velocity and motivates velocity clipping. Experiments show that
our method improves sample quality significantly without pretrained encoders, and that
adapting pretrained encoders and velocity clipping yield further gains. 

\subsection*{AI use statement}

We used generative AI tools to assist with the implementation and editing of
limited portions of the experimental code. All AI-assisted code was reviewed,
verified, and tested by the authors before use. Generative AI tools were not
used to generate synthetic datasets, develop theoretical models or conceptual
frameworks, formulate or prove mathematical claims, design the research
methodology or experiments, perform data analysis, or interpret the results.

We also used generative AI tools to improve the grammar, clarity, and phrasing
of the paper. We reviewed all AI-assisted work and take responsibility for the
final content of this work, including all text, claims, code, results, and
artifacts produced with the aid of generative AI.

\subsection*{Ethics statement}

This work does not involve human subjects or release a new dataset; all
experiments use previously released datasets. As a method for improving
existing generative models, our approach is potentially dual-use and could be
applied to generate misleading, harmful, or inappropriate content. It may
also inherit or amplify biases, privacy risks, and other limitations of the
underlying data and pretrained models. Appropriate safeguards and compliance
with the licenses and intended-use conditions of the relevant models and
datasets should therefore accompany its deployment.

\subsection*{Reproducibility statement}

We provide detailed experimental configurations for reproducibility in
Appendix~\ref{sec:implementation_details}.
Appendix~\ref{app:kde_drift_derivations} provides the derivation of the
matched KDE ratio and drift regression loss gradient identity and discusses
the GAN connection. Appendix~\ref{app:generator_updates_comparison} provides
the generator-update pseudocode.

\begin{ack}
Doudou Zhang, Wenwen Hou, Yilin Chen and Qi Chen are supported by research
funding from the ELLIS Institute Finland and the Department of Computer Science
at Aalto University. We acknowledge CSC-IT Center for Science, Finland, for
providing access to the supercomputers Roihu and LUMI, owned by the European
High Performance Computing Joint Undertaking (EuroHPC JU) and hosted by CSC
Finland in collaboration with the LUMI consortium. We also acknowledge the
computational resources provided by the Aalto Science-IT project through the
Triton cluster.
\end{ack}

\bibliography{iclr2027_conference}
\bibliographystyle{plainnat}

\clearpage
\appendix
\section*{Appendix}
\addcontentsline{toc}{section}{Appendix}

\section{Implementation Details}
\label{sec:implementation_details}

Table~\ref{tab:impl_all_configs} summarizes the configurations and
hyperparameters for the main benchmark (Table~\ref{tab:main_results}) and the
ablation studies (Section~\ref{sec:ablations}).

\begin{table}[h]
\centering
\caption{\textbf{Configurations for the Main Benchmark and Ablation Studies.}}
\label{tab:impl_all_configs}
\scriptsize
\setlength{\tabcolsep}{3pt}
\renewcommand{\arraystretch}{1.05}
\begin{tabular}{@{}p{0.22\linewidth}|C{0.24\linewidth}|C{0.24\linewidth}|C{0.24\linewidth}@{}}
\toprule
& \textbf{Ablation default} & \textbf{Standard drifting} & \textbf{Ours} \\
\midrule
\rowcolor[gray]{0.9} \multicolumn{4}{l}{\textit{\textbf{Generator}}} \\
Architecture & \multicolumn{3}{c}{U-Net} \\
Base channels, mult. & \multicolumn{3}{c}{$128$, $(1,2,2,2)$} \\
Res.\ blocks / attn.\ res.\ & \multicolumn{3}{c}{$2$ per stage; $16{\times}16$} \\
\midrule
\rowcolor[gray]{0.9} \multicolumn{4}{l}{\textit{\textbf{Generator Optimizer}}} \\
Optimizer & \multicolumn{3}{c}{AdamW} \\
Betas & \multicolumn{3}{c}{$(0.9,0.999)$} \\
Weight decay & \multicolumn{3}{c}{$0$} \\
Learning rate & \multicolumn{3}{c}{$2\times10^{-4}$} \\
Batch size (real \& generated) & \multicolumn{3}{c}{$1{,}500$ per GPU} \\
\midrule
\rowcolor[gray]{0.9} \multicolumn{4}{l}{\textit{\textbf{Encoder}}} \\
Architecture & DRN-like (App.~\ref{app:representation_architecture}) & -- & DRN-like \\
\midrule
\rowcolor[gray]{0.9} \multicolumn{4}{l}{\textit{\textbf{Encoder Optimizer}}} \\
Optimizer & AdamW & -- & AdamW \\
Betas & $(0.9,0.999)$ & -- & $(0.9,0.999)$ \\
Weight decay & $0$ & -- & $0$ \\
Learning rate & $10^{-4}$ & -- & $10^{-4}$ \\
Batch size (real \& generated) & $1{,}500$ per GPU & -- & $1{,}500$ per GPU \\
Update schedule & 1 step, per generator step & -- & 1 step, per generator step \\
\midrule
\rowcolor[gray]{0.9} \multicolumn{4}{l}{\textit{\textbf{Loss Computation}}} \\
Mechanism & \multicolumn{3}{c}{stateless batch KDE} \\
KDE temperature & \multicolumn{3}{c}{$0.05$} \\
\midrule
\rowcolor[gray]{0.9} \multicolumn{4}{l}{\textit{\textbf{Regularization}}} \\
$V_{\max}$ & $0.005$ & -- & $0.005$ \\
\midrule
\rowcolor[gray]{0.9} \multicolumn{4}{l}{\textit{\textbf{Training}}} \\
GPUs & \multicolumn{3}{c}{$4\times$ NVIDIA GH200} \\
Training steps & $10{,}000$ & $50{,}000$ & $50{,}000$ \\
\bottomrule
\end{tabular}
\end{table}

\subsection{Additional Configuration Settings}

\textbf{Toy Experiment Setting.} The semantic variable in the nuisance-dimension toy experiment
(Section~\ref{sec:mechanistic_analysis}) follows a ring of eight isotropic
Gaussians with radius $2.0$ and standard deviation $0.08$. Each sample is
paired with an independent 30-dimensional standard Gaussian nuisance vector,
scaled by $1.0$ and concatenated with the two semantic dimensions to
form the 32-dimensional observation used to construct the drift field. The
generator is a three-hidden-layer MLP with 256
hidden units and SiLU activations and produces 32-dimensional samples from
Gaussian noise. The persistent representation uses
an identity-initialized residual MLP with one 64-unit hidden layer,
following the same residual parameterization as the DRN encoder. The
generator and the representation both use AdamW with a learning rate of
$10^{-3}$, weight decay $0$, and a batch size of $256$, updating the
representation once per generator step over $1{,}000$ steps with a KDE
temperature of $0.5$. For the alignment diagnostic in Figure~\ref{fig:toy_field_alignment}(b),
we project the estimated drift onto the two-dimensional semantic space and
compare its direction with a reference field constructed from independent
sets of 4,096 real and generated samples in that space.

\textbf{Persistent Scalar Critic Setting.} The persistent critic uses the SNGAN discriminator
architecture~\citep{miyatoSpectralNormalizationGenerative2018} with 128 base
channels, trained with AdamW using a logistic objective over real and
generated samples, a learning rate of $10^{-4}$, a batch size of
$4\times1{,}500$, one update per generator step, and gradients clipped to a
norm of $1$. Its
scaled output supplies the drift for the drift regression loss through the gradient with
respect to the input, following the same construction used throughout the
paper. In the joint configuration, the critic is applied to the same
DRN-style encoder described in Appendix~\ref{app:representation_architecture},
trained jointly under the same optimizer settings.

\textbf{MoCo-v2 Adaptation Setting.} We initialize a ResNet-50 backbone from the official 800-epoch MoCo v2
checkpoint~\citep{heMomentumContrastUnsupervised2020,chenImprovedBaselinesMomentum2020}
pretrained on ImageNet~\citep{dengImageNetLargescaleHierarchical2009}, resizing
inputs to $112\times112$ with standard ImageNet normalization. The frozen
configuration keeps the backbone fixed, while the adapted configuration
unfreezes only the final residual block, updated with the same persistent
schedule as the DRN encoder using AdamW with a learning rate of $10^{-6}$,
weight decay $0$, a batch size of $4\times1{,}500$, and gradients clipped to
a norm of $1$.

\subsection{Persistent Representation Architecture}
\label{app:representation_architecture}

Persistent representation learning accommodates encoder architectures tailored
to the data modality. We explored several common choices, including ResNet-style
convolutional networks~\citep{heDeepResidualLearning2016},
U-Nets~\citep{ronnebergerUNetConvolutionalNetworks2015}, and Vision
Transformers~\citep{dosovitskiyImageWorth16x162021}. For the main experiments,
we use a lightweight DRN variant. Although simple, this
architecture is sufficient to provide an effective learned representation for
drifting.

Our DRN-style encoder preserves the input resolution throughout the network
and uses the residual parameterization
$E_\varphi(x)=x+R_\varphi(x)$, illustrated in
Figure~\ref{fig:persistent_residual_encoder}. The residual branch begins with a $3\times3$
convolution from three input channels to 128 hidden channels, followed by four
constant-width residual blocks with dilation rates $(1,2,4,1)$ and SiLU
activations. Each block contains two $3\times3$ convolutions and an identity
skip connection. A final $3\times3$ projection maps the hidden features back
to three channels. We use neither spatial downsampling nor normalization
layers. Initializing the final projection to zero gives
$E_{\varphi_0}(x)=x$ exactly, while subsequent training modifies the
representation through the learned residual.
As in Drifting
Models~\citep{dengGenerativeModelingDrifting,falahatiDriftXpressFasterDrifting2026}, the drifting objective can
also extend to multiple feature groups, which we extract from the input, the
stem output, each residual-block output, and the final encoded image.

\begin{figure}[!t]
\centering
\begin{tikzpicture}[
  >=Latex,
  every node/.style={font=\normalsize},
  block/.style={draw, rounded corners, minimum width=34mm,
    minimum height=10mm, align=center},
  sum/.style={draw, circle, inner sep=0pt, minimum size=7mm}
]
\node (input) at (0,3.45) {$x$};
\coordinate (fork) at (0,3.05);
\node[block] (residual) at (0,1.95)
  {Trainable mapping\\$R_\varphi$};
\node[sum] (sum) at (0,0.75) {$+$};
\node (output) at (0,-0.20) {$E_\varphi(x)$};
\draw[->] (input) -- (fork);
\fill (fork) circle (1.2pt);
\draw[->] (fork) -- (residual);
\draw[->] (residual) -- (sum);
\draw[->] (fork) -- (2.20,3.05) -- (2.20,0.75) -- (sum);
\draw[->] (sum) -- (output);
\end{tikzpicture}
\caption{\textbf{Persistent Representation Learning with a Residual
Parameterization.}}
\label{fig:persistent_residual_encoder}
\end{figure}
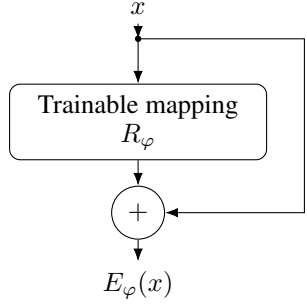

\newpage

\section{Additional Analyses}

\subsection{Evolution of the Learned Representation Geometry}
\label{app:learned_geometry_evolution}

To study the evolution of the learned representation geometry, we analyze
checkpoints of the persistent encoder trained from scratch on CIFAR-10~\citep{krizhevskyLearningMultipleLayers2009}. We
compare their representations using linear centered kernel alignment
(CKA)~\citep{kornblithSimilarityNeuralNetwork2019}. CKA measures
scale-invariant similarity between the
pairwise sample relations induced by two encoders, with higher values
indicating greater similarity.

\begin{figure}[h!]
\centering
\includegraphics[width=0.5\textwidth]{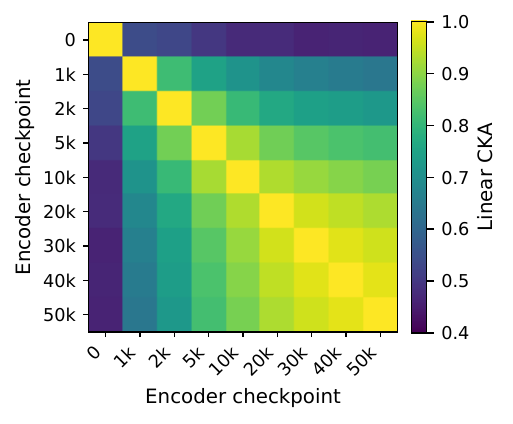}
\caption{\textbf{Representation Similarity across Training.}
Linear CKA between encoder checkpoints shows a gradual departure from the
identity initialization that stabilizes among late-training checkpoints.}
\label{fig:encoder_geometry_cka}
\end{figure}

As shown in Figure~\ref{fig:encoder_geometry_cka}, persistent representation
learning produces a gradual change in the geometry used to
construct the drifting field. The learned representation progressively
departs from its identity initialization, with CKA against the initial
encoder decreasing to $0.46$ at 50k steps, while neighboring late-training
checkpoints remain highly similar, with CKA above $0.95$ from 30k onward.

\begin{figure}[h]
\centering
\includegraphics[width=0.55\textwidth]{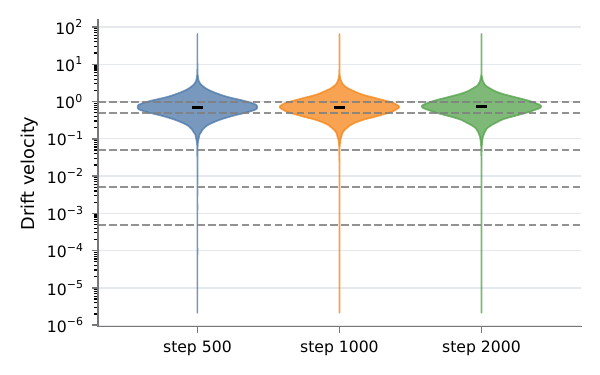}
\caption{\textbf{Velocity Budgets Relative to Drift Magnitudes.} Distribution of
$\|\widehat V_\varphi(u_i)\|_2$ at steps 500, 1000, and 2000 of the $V_{\max}=0.005$
run. Dashed lines mark the five $V_{\max}$ budgets compared in
Table~\ref{tab:velocity_budget_ablation}.}
\label{fig:velocity_clip_before_clip}
\end{figure}

\subsection{Drift-Magnitude Distributions and Velocity Budgets}
\label{app:velocity_clip_before_clip}

Figure~\ref{fig:velocity_clip_before_clip} presents the distribution of the
drift magnitude $\|\widehat V_\varphi(u_i)\|_2$
(Eq.~(\ref{eq:l2_drift_projection})) immediately before clipping, measured at
training steps 500, 1000, and 2000 of the $V_{\max}=0.005$ run. The dashed
lines indicate the five velocity budgets evaluated in
Table~\ref{tab:velocity_budget_ablation}. The three distributions have similar
shapes, with medians remaining approximately $0.70$ across steps 500, 1000,
and 2000. The $V_{\max}=1$ and $V_{\max}=0.5$ lines intersect the
displayed distributions, whereas the budgets at or below $0.05$ lie beneath
their main mass.

\subsection{Kernel Bandwidth in Pixel Space}
\label{app:pixel_bandwidth_sweep}

Changing the KDE temperature alone cannot make Drifting Models train well in
pixel space. We sweep the temperature $\tau$ over three orders of magnitude
($0.002$ to $2$) in pixel space, training each configuration for 10k steps
on CIFAR-10.
As shown in Figure~\ref{fig:pixel_bandwidth_sweep}, FID remains between
$134$ and $376$ throughout training across the sweep, far above the $27.41$
achieved by persistent representation learning within the same step budget.
This indicates that the ineffectiveness of pixel-space drifting reflects the
representation geometry itself, not merely a poorly tuned kernel bandwidth.

\begin{figure}[!t]
\centering
\includegraphics[width=0.5\textwidth]{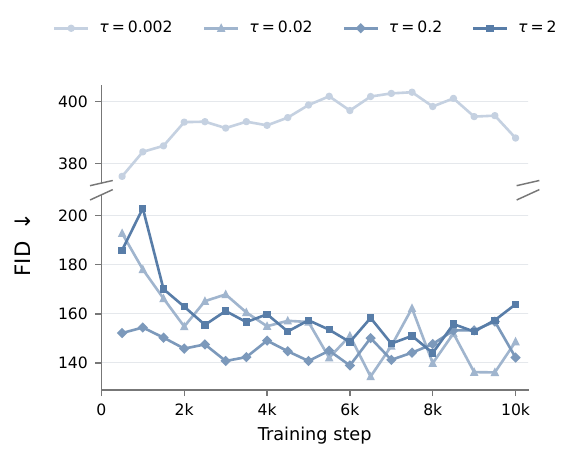}
\caption{\textbf{Effect of Kernel Bandwidth on Pixel-space Drifting.} FID over
training steps on CIFAR-10 for four KDE temperatures
$\tau\in\{0.002, 0.02, 0.2, 2\}$ in pixel space. None approaches the $27.41$
FID achieved by persistent representation learning within the same 10k-step
budget.}
\label{fig:pixel_bandwidth_sweep}
\end{figure}

\subsection{Training Time Cost Comparison}
\label{app:training_cost}

We compare standard drifting and our method under the same training time
budget. Persistent representation learning adds an encoder update and the
encoder forward and backward passes to each training step. Figure~\ref{fig:cifar10_training_time} reports FID against training time on CIFAR-10~\citep{krizhevskyLearningMultipleLayers2009}, and with four GH200 GPUs, our method 
reaches an FID of $21.57$, compared with $124.94$ for standard drifting.

\begin{figure}[!t]
\centering
\includegraphics[width=0.5\textwidth]{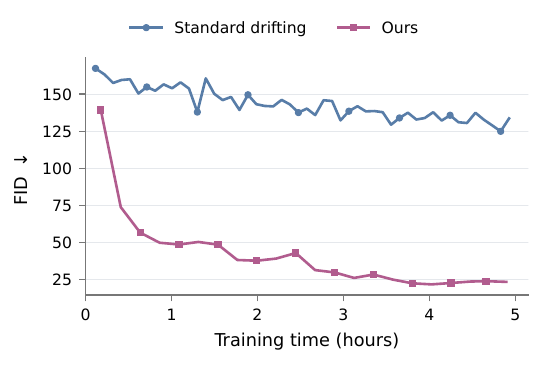}
\caption{\textbf{FID under the Same Training Time on CIFAR-10.} FID against
training time within 5 hours.}
\label{fig:cifar10_training_time}
\end{figure}

\clearpage

\section{Additional Qualitative Results}
\label{app:additional_qualitative_results}

Figure~\ref{fig:additional_svhn_samples} provides additional samples from our
method on SVHN~\citep{netzerReadingDigitsNatural2011}.

\begin{figure*}[h]
\centering
\includegraphics[width=\textwidth]{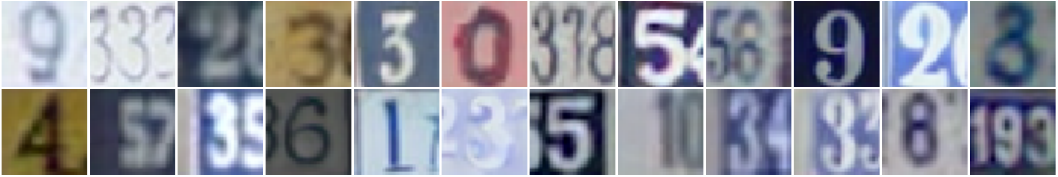}
\caption{\textbf{Additional Qualitative Results on SVHN.} Samples generated by
our method at the checkpoint attaining the best FID of $2.88$.}
\label{fig:additional_svhn_samples}
\end{figure*}

\clearpage

\section{KDE Ratio Loss, Drift Regression Loss, and the GAN loss Connection}
\label{app:kde_drift_derivations}

\subsection{Proof of Proposition~\texorpdfstring{\ref{prop:kde_drifting_gradient}}{1}}
\label{app:proof_kde_drifting}

\begingroup
\renewcommand{\theproposition}{\ref{prop:kde_drifting_gradient}}
\begin{proposition}[Equivalence of representation-space generator gradients, restated]
Let $B\ge1$ and $\tau>0$. For each sampled input $z_i$, let
$\hat{x}_i=G_\theta(z_i)$ and $u_i=E_\varphi(\hat{x}_i)$.
Assume that $G_\theta$ is differentiable with respect to $\theta$ at each
$z_i$ and that $E_\varphi$ is differentiable with respect to its input at
each $\hat{x}_i$. Write $J_E(x)=\partial E_\varphi(x)/\partial x$ and
$J_G(z)=\partial G_\theta(z)/\partial\theta$.
Use the Gaussian kernel in Eq.~(\ref{eq:gaussian_kernel}), with the same
bandwidth and feature anchors in the KDE ratio loss
$\widehat{\mathcal L}_{\varphi,\theta}^{\text{KDE}}$ and the drift regression loss
$\widehat{\mathcal L}_{\varphi,\theta}^{\text{Drift}}$ defined in Section~\ref{sec:density_ratio_drift}.
Hold $\varphi$ and the feature anchors fixed, differentiating the KDE ratio
loss only through the generated queries $E_\varphi(G_\theta(z_i))$.
With no velocity clipping, the gradients evaluated at the current parameters
used to construct the detached regression targets satisfy
\begin{equation*}
\nabla_\theta\widehat{\mathcal L}_{\varphi,\theta}^{\text{KDE}}
=\nabla_\theta\widehat{\mathcal L}_{\varphi,\theta}^{\text{Drift}}
=-\frac{2}{B}\sum_{i=1}^B
J_G(z_i)^\top J_E(\hat{x}_i)^\top\widehat V_\varphi(u_i).
\end{equation*}
\end{proposition}
\endgroup

\begin{proof}
For $\tau>0$, both representation-space Gaussian KDEs are positive.
Let $v_j=E_\varphi(\tilde{x}_j)$ denote the real feature anchors.
Hold the anchor values $\{u_r\}_{r=1}^B$ and $\{v_j\}_{j=1}^B$
fixed; the generated query $u_i=E_\varphi(G_\theta(z_i))$ remains
differentiable through $G_\theta$ and $E_\varphi$. Differentiating
with respect to a feature query $u$ gives
\begin{align*}
\nabla_u\log\widehat p_{\varphi,\tau}(u)
&=\frac{2}{\tau}
\frac{\sum_j k_\tau(u,v_j)(v_j-u)}{\sum_j k_\tau(u,v_j)},\\
\nabla_u\log\widehat q_{\varphi,\tau}^{\theta}(u)
&=\frac{2}{\tau}
\frac{\sum_r k_\tau(u,u_r)(u_r-u)}{\sum_r k_\tau(u,u_r)}.
\end{align*}
Consequently,
\begin{equation}
\tau\nabla_u\log\frac{\widehat p_{\varphi,\tau}(u)}
{\widehat q_{\varphi,\tau}^{\theta}(u)}=2\widehat V_\varphi(u).
\label{eq:density_ratio_drift}
\end{equation}
With the Jacobians defined in the proposition,
$\partial u_i/\partial\theta=J_E(\hat{x}_i)J_G(z_i)$.
Applying the chain rule to the KDE ratio loss and using
Eq.~(\ref{eq:density_ratio_drift}) yields
\begin{equation*}
\nabla_\theta\widehat{\mathcal L}_{\varphi,\theta}^{\text{KDE}}
=-\frac{2}{B}\sum_{i=1}^B
J_G(z_i)^\top J_E(\hat{x}_i)^\top
\widehat V_\varphi(u_i).
\end{equation*}
For the detached target
$\widetilde u_i=\operatorname{sg}(u_i+\widehat V_\varphi(u_i))$,
the target is held constant during differentiation. At the current
parameters used to construct it, the residual is $u_i-\widetilde u_i=-\widehat V_\varphi(u_i)$.
Differentiating the drift regression loss therefore gives
\begin{equation*}
\nabla_\theta\widehat{\mathcal L}_{\varphi,\theta}^{\text{Drift}}
=\frac{2}{B}\sum_{i=1}^B
J_G(z_i)^\top J_E(\hat{x}_i)^\top
(u_i-\widetilde u_i)
=\nabla_\theta\widehat{\mathcal L}_{\varphi,\theta}^{\text{KDE}}.
\end{equation*}
This proves the claimed identity under the stated differentiability assumptions.
Repeated anchors do not invalidate the argument because the KDEs remain
positive. The encoder parameters and anchors are fixed only for this
generator-gradient comparison; the encoder-learning phase differentiates
through the KDE to update $\varphi$.
\end{proof}

\textbf{Scope of the identity.}
The two objectives must use the same anchors, bandwidth, and treatment of
self-interactions. Differentiating through generated anchors introduces additional terms in the KDE ratio loss. Clipping changes the drift regression
loss gradient and generally breaks equality with the KDE ratio loss gradient. The proposition concerns one current-step gradient, and does not
assert an equality of losses or entire training trajectories.

\subsection{Pixel-Space Special Case and Comparison}
\label{app:pixel_kde_drifting}

Setting $E_\varphi=\mathrm{Id}$ gives $u_i=\hat{x}_i$,
$v_j=\tilde{x}_j$, and $J_E=I$. We retain the original pixel-space
notation below.
We denote the batch size by \(B\). For each batch, let
\(\{\tilde{x}_j\}_{j=1}^B\) denote the real samples and
\(\{\hat{x}_i\}_{i=1}^B\) the generated samples, where
$\hat{x}_i = G_\theta(z_i), z_i \sim \mu, i \in [B].$ The KDEs estimated from the current batch:
\begin{equation*}
\widehat p_\tau(x)=\frac{1}{B}\sum_{j=1}^B k_\tau(x,\tilde{x}_j),
\qquad
\widehat q^{\theta}_\tau(x)=\frac{1}{B}\sum_{i=1}^B k_\tau(x,\hat{x}_i).
\end{equation*}
The dependence on the current generator is implicit.
The corresponding density ratio and scalar potential are
\begin{equation*}
\widehat r(x)
=\frac{\widehat p_\tau(x)}{\widehat q^{\theta}_\tau(x)},
\qquad
\widehat U(x)=\tau\log\widehat r(x).
\end{equation*}

In this case, $\tau s_\varphi^{\mathrm{KDE}}(x)=\widehat U(x)$.
We retain the loss notation from Section~\ref{sec:density_ratio_drift},
with $E_\varphi=\mathrm{Id}$ throughout this pixel-space specialization.
The KDE ratio loss and drift regression loss therefore become
\begin{align*}
\widehat{\mathcal L}_{\varphi,\theta}^{\text{KDE}}
=-\frac{1}{B}\sum_{i=1}^B\widehat U(G_\theta(z_i)),\quad
\widehat{\mathcal L}_{\varphi,\theta}^{\text{Drift}}
=\frac{1}{B}\sum_{i=1}^B
\left\|G_\theta(z_i)-\operatorname{sg}\!\left(
\hat{x}_i+\widehat V_{p,q_\theta}(\hat{x}_i)\right)\right\|_2^2.
\end{align*}
Let $J_i=\partial G_\theta(z_i)/\partial\theta$.
For $\tau>0$, the Gaussian KDEs are positive, so their log gradients are
well defined. Holding the anchors fixed and differentiating the query gives
\begin{align*}
\nabla_x\log\widehat p_\tau(x)
&=\frac{2}{\tau}
\frac{\sum_{j=1}^{B}k_\tau(x,\tilde{x}_j)(\tilde{x}_j-x)}
{\sum_{j=1}^{B}k_\tau(x,\tilde{x}_j)},\\
\nabla_x\log\widehat q^{\theta}_\tau(x)
&=\frac{2}{\tau}
\frac{\sum_{i=1}^{B}k_\tau(x,\hat{x}_i)(\hat{x}_i-x)}
{\sum_{i=1}^{B}k_\tau(x,\hat{x}_i)}.
\end{align*}
Subtracting these expressions and using the definitions of $\widehat U$
and $\widehat V_{p,q_\theta}$ yields
\begin{equation}
\nabla_x\widehat U(x)
=\tau\left(\nabla_x\log\widehat p_\tau(x)
-\nabla_x\log\widehat q^{\theta}_\tau(x)\right)
=2\widehat V_{p,q_\theta}(x).
\label{eq:pixel_density_ratio_drift}
\end{equation}
Although $\widehat q^{\theta}_\tau$ records the generator that produced the
anchors, those anchors remain constant in the differentiation below.
The chain rule therefore gives
\begin{equation*}
\nabla_\theta\widehat{\mathcal L}_{\varphi,\theta}^{\text{KDE}}
=-\frac{1}{B}\sum_{i=1}^{B}J_i^\top\nabla_x\widehat U(\hat{x}_i)
=-\frac{2}{B}\sum_{i=1}^{B}J_i^\top
\widehat V_{p,q_\theta}(\hat{x}_i).
\end{equation*}
For the drift regression loss, the stop-gradient target is constant.
At the current parameters, its residual is
$G_\theta(z_i)-\operatorname{sg}(\hat{x}_i+
\widehat V_{p,q_\theta}(\hat{x}_i))=-\widehat V_{p,q_\theta}(\hat{x}_i)$.
Thus
\begin{equation*}
\nabla_\theta\widehat{\mathcal L}_{\varphi,\theta}^{\text{Drift}}
=-\frac{2}{B}\sum_{i=1}^{B}J_i^\top
\widehat V_{p,q_\theta}(\hat{x}_i)
=\nabla_\theta\widehat{\mathcal L}_{\varphi,\theta}^{\text{KDE}},
\end{equation*}
which recovers the original pixel-space gradient identity. The argument also covers coincident anchors,
since the Gaussian KDEs remain strictly positive.

\textbf{What changes with the representation?}
Pixel-space kernels use $\|x-y\|_2^2$, whereas representation-space kernels
use $\|E_\varphi(x)-E_\varphi(y)\|_2^2$. Thus an encoder changes which
anchors contribute to the KDE scores and drift. In addition, the feature
direction is pulled back through $J_E^\top$ before reaching the generator;
this factor is the identity in pixel space. These are the two mechanisms by
which learning the representation changes generator supervision. The
comparison establishes no universal ordering of the two geometries:
an encoder that discards relevant differences can also remove useful
supervision. The benefits of the learned geometry are evaluated empirically
in Figures~\ref{fig:toy_field_alignment} and
\ref{fig:representation_fid_kde_diagnostics}.

\subsection{Connection to Density-Ratio GAN Objectives}
\label{app:gan_critic_connection}

\textbf{Divergence-Based GAN Objectives.} Divergence-based GANs~\citep{nowozinFGANTrainingGenerative2016,yiMonoFlowRethinkingDivergence2023}
formulate distribution matching as a minimax problem between a generator
$G_\theta$ and a discriminator $D:\Reals^d\rightarrow\Ical$, where $\Ical\subseteq\Reals$ is the
range of discriminator outputs. We use a variational
representation of a divergence between $p$ and $q_\theta$, and let $\phi,\psi:\Ical\to\Reals$
be differentiable scalar functions specifying the discriminator's objective
contributions for real and generated samples, respectively.
For example, the standard logistic GAN uses
$\phi(t)=\log\sigma(t)$ and $\psi(t)=\log(1-\sigma(t))$,
where $\sigma(t)=(1+e^{-t})^{-1}$.
The corresponding minimax objective is
\begin{equation*}
\min_\theta\max_{D}
\left\{
\mathbb E_{\tilde{x}\sim p}[\phi(D(\tilde{x}))]
+\mathbb E_{z\sim\mu}[\psi(D(G_\theta(z)))]
\right\}.
\end{equation*}
For a fixed generator $G_\theta$, let $D^*$ denote a maximizer of the
inner problem. In practice, the generator can also be trained with an alternative loss
against a fixed discriminator 
$
\min_\theta\ -\mathbb E_{z\sim\mu}[g(D(G_\theta(z)))],
$
where $g:\Ical\to\Reals$ is a differentiable increasing scalar function.
Choosing $g=-\psi$ recovers the generator update of the minimax objective.
For the logistic discriminator, $g(t)=\log\sigma(t)$ gives the
non-saturating generator loss.

For an unrestricted discriminator class, an interior optimum $D^*$
satisfies the pointwise first-order condition, yielding the following
identity wherever $q_\theta(x)>0$ and $\phi'(D^*(x))\ne0$:
\begin{equation}
-\psi'(D^*(x))/\phi'(D^*(x))
=p(x)/q_\theta(x).
\label{eq:optimal_ratio}
\end{equation}
Thus, the optimal discriminator output can be transformed into the density
ratio between the real and generated data distributions. During the generator
update, $D^*$ is held fixed, and the input gradient of
$D^*(G_\theta(z))$ supplies a distribution-matching direction.

The optimality condition in Eq.~(\ref{eq:optimal_ratio}) shows that the output
of a discriminator $D$ induces the density-ratio estimate
$r_D(x):=-\psi'(D(x))/\phi'(D(x))$. At the function-space optimum,
$r_{D^*}(x)=p(x)/q_\theta(x)$. For a divergence parameterization satisfying
$r_D(x)>0$, define the scalar potential $U_D(x)=h(\log r_D(x))$,
where $h:\mathbb R\rightarrow\mathbb R$ is differentiable and strictly
increasing. At the optimum,
\begin{align*}
U_{D^*}(x)&=h\!\left(\log\frac{p(x)}{q_\theta(x)}\right),\qquad
\nabla_xU_{D^*}(x)
=h'\!\left(\log\frac{p(x)}{q_\theta(x)}\right)
\left(\nabla\log p(x)-\nabla\log q_\theta(x)\right).
\end{align*}

The potential $U_D(x)$ assigns higher values to regions where
the estimated data density is large relative to the current
model density. Holding $D$ fixed, we train the generator to
increase this value at its outputs by minimizing $\mathcal L_G(\theta;D)
=-\mathbb E_{z\sim\mu}\left[U_D(G_\theta(z))\right].$ The gradient $\nabla_x U_D(x)$ is backpropagated through
$G_\theta$ to update its parameters, encouraging the generator
to allocate more probability mass to regions that are
underrepresented relative to the data.

Choosing $h(a)=\tau a$ and replacing the population ratio with the
Gaussian KDE ratio $\widehat r$ gives the potential $\widehat U$ in
Appendix~\ref{app:pixel_kde_drifting}. Proposition~\ref{prop:kde_drifting_gradient}
and its pixel-space specialization establish the exact matched-KDE gradient identity. For a general
$h$, the factor $h'(\log r_D(x))$ changes the local magnitude; a learned
discriminator also need not equal the empirical KDE ratio. The proposition
therefore does not assert equivalence with arbitrary GAN generator losses.
The matched batch is used to compare the two generator objectives, while
Algorithm~\ref{alg:trainable_encoder_drifting} uses fresh batches for the
encoder and generator phases.

\textbf{Connection to velocity control.}
Spectral normalization controls a discriminator's Lipschitz constant by
constraining the spectral norm of each layer and is widely used to stabilize
GAN training~\citep{miyatoSpectralNormalizationGenerative2018}. Motivated by
this principle, we examine the corresponding controlled quantity in Drifting
Models. The Gaussian-kernel KDE correspondence established in
Section~\ref{sec:density_ratio_drift} relates the input gradient of a scalar
potential to the explicit drifting velocity. Under the logistic density-ratio
parameterization, $r_D(x)=e^{D(x)}$. Taking $h(u)=\tau u$ gives
$U_D(x)=\tau D(x)$; within this specialization, the associated drifting
velocity satisfies
$V(x)=\frac{1}{2}\nabla_xU_D(x)=\frac{\tau}{2}\nabla_xD(x)$. If $D$ is
$L_D$-Lipschitz with respect to the input norm, then
$\|\nabla_xD(x)\|_2\le L_D$ almost everywhere and consequently
$\|V(x)\|_2\le\tau L_D/2$.
Thus, a Lipschitz constraint on the discriminator corresponds to a bound on
particle velocity in the associated drifting formulation.

\section{A Detailed Comparison of Generator Updates}
\label{app:generator_updates_comparison}

The following algorithms make explicit the optimization sequences compared in
Section~\ref{sec:density_ratio_drift}. Drifting directly constructs a vector field and uses the drift regression loss to fit a detached target. A divergence-based GAN first optimizes
a discriminator and then updates the generator through the density-ratio
potential induced by the fixed discriminator.
Here the KDE ratio loss $\widehat{\mathcal L}_{\varphi,\theta}^{\text{KDE}}$
and drift regression loss $\widehat{\mathcal L}_{\varphi,\theta}^{\text{Drift}}$
use the definitions in Section~\ref{sec:density_ratio_drift}, specialized to
$E_\varphi=\mathrm{Id}$ as in Appendix~\ref{app:pixel_kde_drifting}.
We write $J_i=J_G(z_i)=\partial G_\theta(z_i)/\partial\theta$ in both algorithms.
The general GAN objective $\widehat{\mathcal L}_G(D^*,G_\theta)$ is denoted
separately from the KDE ratio loss.

\begin{figure}[H]
\noindent\begin{minipage}{0.8\textwidth}
\vspace{0pt}
\hrule
\captionof{algorithm}{\textbf{Drifting Models Training}}
\label{alg:drifting_training}
\vspace{2pt}
\hrule
\vspace{2pt}
\small
\textbf{Input:} Generator $G_\theta$, real data distribution $p$, latent
distribution $\mu$, kernel $k_\tau$, batch size $B$, and learning rate
$\gamma_G$
\begin{algorithmic}[1]
\Repeat
  \State Sample $\{z_i\}_{i=1}^{B}\sim\mu$ and $\{\tilde{x}_j\}_{j=1}^{B}\sim p$
  \State $\hat{x}_i\gets G_\theta(z_i)$ for all $i$
  \State Construct $\widehat V_{p,q_\theta}(\hat{x}_i)$ using Eq.~(\ref{eq:empirical_drift})
  \State $x_i^{\mathrm{target}}\gets\operatorname{sg}(\hat{x}_i+\widehat V_{p,q_\theta}(\hat{x}_i))$
  \State $\widehat{\mathcal L}_{\varphi,\theta}^{\text{Drift}}\gets B^{-1}\sum_{i=1}^{B}\|\hat{x}_i-x_i^{\mathrm{target}}\|_2^2$
  \State $\nabla_\theta\widehat{\mathcal L}_{\varphi,\theta}^{\text{Drift}}=-2B^{-1}\sum_{i=1}^{B}J_i^\top\widehat V_{p,q_\theta}(\hat{x}_i)$
  \State $\theta\gets\theta-\gamma_G\nabla_\theta\widehat{\mathcal L}_{\varphi,\theta}^{\text{Drift}}$
\Until{converged}
\end{algorithmic}
\vspace{2pt}
\hrule
\end{minipage}
\end{figure}

\begin{figure}[H]
\noindent\begin{minipage}{0.8\textwidth}
\vspace{0pt}
\hrule
\captionof{algorithm}{\textbf{Divergence-Based GAN and Matched KDE Update}}
\label{alg:gan_training}
\vspace{2pt}
\hrule
\vspace{2pt}
\small
\textbf{Input:} Generator $G_\theta$, discriminator $D$, real data distribution
$p$, latent distribution $\mu$, objectives $\phi,\psi,h$, kernel $k_\tau$,
batch size $B$, discriminator steps $K_D$, and learning rates
$\gamma_D,\gamma_G$
\begin{algorithmic}[1]
\Repeat
  \For{$k=1,\ldots,K_D$}
    \State Sample $\{z_i\}_{i=1}^{B}\sim\mu$ and
      $\{\tilde{x}_j\}_{j=1}^{B}\sim p$
    \State $\hat{x}_i\gets\operatorname{sg}(G_\theta(z_i))$ for all $i$
    \State $\displaystyle \widehat{\mathcal V}(D,G_\theta)\gets
      \frac{1}{B}\sum_{j=1}^{B}\phi(D(\tilde{x}_j))
      +\frac{1}{B}\sum_{i=1}^{B}\psi(D(\hat{x}_i))$
    \State $D\gets D+\gamma_D\nabla_D
      \widehat{\mathcal V}(D,G_\theta)$
  \EndFor
  \State Reuse $\{z_i\}$ and $\{\tilde{x}_j\}$ of Algorithm~\ref{alg:drifting_training}; $\hat{x}_i\gets G_\theta(z_i)$
  \State Estimate $r_{D^*}=p/q_\theta$ by $\widehat r=\widehat p_\tau/\widehat q^{\theta}_\tau$ with fixed anchors
  \State $\widehat{\mathcal L}_G(D^*,G_\theta)\gets-B^{-1}\sum_{i}h(\log\widehat r(\hat{x}_i))=\widehat{\mathcal L}_{\varphi,\theta}^{\text{KDE}}$
  \State $\nabla_\theta\widehat{\mathcal L}_{\varphi,\theta}^{\text{KDE}}=-2B^{-1}\sum_{i}J_i^\top\widehat V_{p,q_\theta}(\hat{x}_i)=\nabla_\theta\widehat{\mathcal L}_{\varphi,\theta}^{\text{Drift}}$
  \State $\theta\gets\theta-\gamma_G\nabla_\theta\widehat{\mathcal L}_G(D^*,G_\theta)$
  \color{black}
\Until{converged}
\end{algorithmic}
\vspace{2pt}
\hrule
\end{minipage}
\end{figure}

The matched KDE comparison uses the same real and generated batch,
fixed anchors, $h(a)=\tau a$, and no velocity clipping; the gradient identity
is derived in Appendix~\ref{app:pixel_kde_drifting}. The connection to a
GAN critic is discussed separately in Appendix~\ref{app:gan_critic_connection}.
This comparison isolates the current-step gradients, whereas
Algorithm~\ref{alg:trainable_encoder_drifting} uses fresh batches for
the representation and generator phases.

\section{Limitations}
\label{sec:limitations}

Our study has two limitations. Although representations learned from scratch
substantially improve standard drifting, their performance still trails that
obtained with pretrained encoders, indicating that training dynamics alone do
not yet recover all of the useful structure provided by external pretraining.
In addition, owing to computational constraints, the effectiveness of
persistent representation learning on larger-scale datasets remains to be
established.

\end{document}